\documentclass[journal]{IEEEtran}

\usepackage{graphics} 
\usepackage{epsfig} 
\usepackage{mathptmx} 
\usepackage{times} 
\usepackage{amsmath} 
\usepackage{amssymb}  
\usepackage{cite}
\usepackage{bm}
\usepackage{color}
\usepackage{arcs}
\usepackage{algorithm}               
\usepackage{algorithmic}             

\newtheorem{defn}{Definition}
\newtheorem{rem}{Remark}
\newtheorem{thm}{\textbf{Theorem}}

\newcommand*{\QEDA}{\hfill\ensuremath{\square}}

\begin{document}

\title{Real-time Acceleration-continuous Path-constrained Trajectory Planning With Built-in Tradability Between Cruise and Time-optimal Motions }

\author{Peiyao Shen, Xuebo Zhang, and Yongchun Fang
\thanks{This work is supported in part by National Natural Science Foundation of China under Grants 61573195 and U1613210. (\emph{Corresponding author: Xuebo Zhang})}
\thanks{The authors are with the Institute of Robotics and Automatic Information System, Nankai University, Tianjin 300071, China,
and also with the Tianjin Key Laboratory of Intelligent Robotics, Nankai University, Tianjin 300071, China.
        {\tt\small (e-mail: zhangxuebo@nankai.edu.cn)}}%
}

\markboth{~}%
{Shell \MakeLowercase{\textit{et al.}}: Bare Demo of IEEEtran.cls for IEEE Journals}

\maketitle

\begin{abstract}
In this paper, a novel real-time acceleration-continuous path-constrained trajectory planning algorithm is proposed with an appealing built-in tradability mechanism between cruise motion and time-optimal motion. Different from existing approaches, the proposed approach smoothens time-optimal trajectories with bang-bang input structures to generate acceleration-continuous trajectories while preserving the completeness property. More importantly, a novel built-in tradability mechanism is proposed and embedded into the trajectory planning framework, so that the proportion of the cruise motion and time-optimal motion can be flexibly adjusted by changing a user-specified functional parameter. Thus, the user can easily apply the trajectory planning algorithm for various tasks with different requirements on motion efficiency and cruise proportion. Moreover, it is shown that feasible trajectories are computed more quickly than optimal trajectories. Rigorous mathematical analysis and proofs are provided for these aforementioned results. Comparative simulation and experimental results on omnidirectional wheeled mobile robots demonstrate the capability of the proposed algorithm in terms of flexible tunning between cruise and time-optimal motions, as well as higher computational efficiency.
\end{abstract}

\begin{IEEEkeywords}
Trajectory planning, Time optimality, Cruise motion, Tradability mechanism
\end{IEEEkeywords}

\IEEEpeerreviewmaketitle

\section{Introduction}

\IEEEPARstart{M}{otion} planning aims to generate an effective collision-free motion law under kinematic and dynamic constraints of robotic systems \cite{2006_Mech}, which can be decoupled into path planning and path-constrained trajectory planning. Geometric paths are generated by path planning under constraints of curvature, collision avoidance and so on \cite{2010_TRO}. Path-constrained trajectory planning algorithms take the generated paths, kinematic and dynamic constraints as input arguments to output feasible velocity profiles along given paths. In this paper, we focus on path-constrained trajectory planning while simultaneously considering cruise motion proportion, time-optimal motion, as well as the algorithmic completeness and its real-time performance.

\begin{figure}[t]
      \centering
      \includegraphics[scale=0.5]{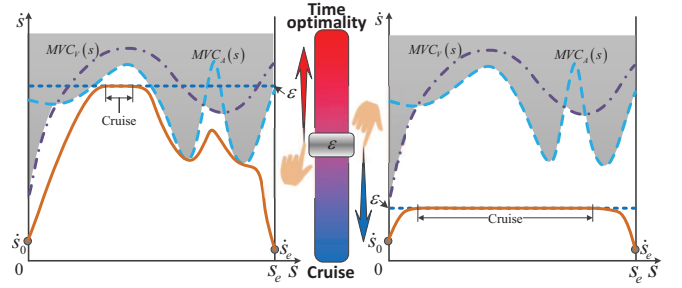}
      \caption{It is a diagram of the proposed algorithm with a built-in tradability mechanism between cruise motion and time-optimal motion. The notations $s,\dot s, \dot s_0, \dot s_e$ represent the path coordinate, path velocity, starting velocity and terminal velocity, respectively. The cyan dash curve $MVC_A(s)$ and purple dash-dot curve $MVC_V(s)$ decide the maximum velocity curve. The orange solid curves represent feasible trajectories generated by the proposed algorithm. In the tradability mechanism, a user-specified functional parameter $\epsilon$ corresponding to the gray-white button decides the adjustable velocity limit line (blue small dash lines). When the button slides to `Time-optimality' (the parameter increasing), the traveling time and cruise proportion of the generated trajectories decrease. When the button slides to `Cruise' (the parameter decreasing), the traveling time and cruise proportion of the generated trajectories increase.}
      \label{fig_intro}
\end{figure}

In order to increase the production efficiency of robotic systems, many trajectory planning techniques design the traveling time as the objective function to generate time-optimal trajectories \cite{1985_TAC,1985_IJRR}. However, their input structures are generally bang-bang, resulting in discontinuous acceleration. Thus, the time-optimal trajectory planning indicates input saturation of at least one actuator at any time instant, such that there is no space for the controller to correct tracking errors caused by disturbances or modeling errors \cite{2006_Mech}. Accordingly, smooth trajectory planning techniques are given to improve tracking accuracy \cite{2013_TRO,2000_TIE,1994_ICRA,2011_JIRS,2014_JFR}. In general, feasible trajectories are expressed with piecewise polynomial interpolation to guarantee smoothness (such as continuous acceleration), and then existing optimization solvers are employed to compute optimal trajectories in polynomial parametric space. However, it is hard to obtain optimal trajectories in global space other than parametric space \cite{2003_TRA,2016_JCAS}. Moreover, most of smooth techniques are achieved in the offline manner due to the nonlinear optimization solvers. The piecewise polynomial trajectories mainly consist of acceleration and deceleration motions, thus it is difficult to possess a high proportion of cruise motion \cite{2017_RCIM}. In this case, the robot is in `go fast or go slow' motion, which may bring potential accidents for urban transport systems \cite{2016_TITS}. In addition, these aforementioned literatures take no notice of an appealing planning property `completeness': a solution is returned if a planning problem is solvable otherwise a failure is returned in finite time. Thus, there is still a lack of a real-time acceleration-continuous path-constrained trajectory planning considering the cruise motion, time-optimality and completeness.

In this paper, we propose a novel real-time acceleration-continuous trajectory planning algorithm with a built-in tradability mechanism between cruise motion and time-optimal motion. Moreover, the proposed algorithm still preserves the completeness property after smoothing the time-optimal trajectories computed by the work \cite{2017_Mech}. By adjusting a user-specified functional parameter in the built-in mechanism, the proposed algorithm outputs an adjustable trajectory between smoothened time-optimal trajectories and feasible trajectories possessing a high proportion of cruise motion. Specifically, in the mechanism, the traveling time and cruise proportion of feasible trajectories generated by the proposed algorithm are inversely in proportional to the functional parameter. In addition, the computational time of the proposed algorithm is in proportional to the functional parameter.


The main procedure of the proposed algorithm is described as follows. First, along the given path, the kinematic and dynamic constraints of the robotic system are transformed into path acceleration constraints. Then, a maximum velocity curve (MVC) is computed to satisfy the path acceleration constraints. With the aid of complete numerical integration techniques \cite{2017_Mech}, a time-optimal trajectory is obtained as a combination of maximum accelerating and maximum decelerating curves under the MVC. However, the path acceleration of the time-optimal trajectory is discontinuous at the intersection points. Thus, a bidirectional integration operation is newly designed in this paper to guarantee continuous path acceleration at intersection points. Specifically, two feasible velocity profiles are computed with continuous path acceleration. In this paper, it is proven that these two velocity profiles joint together and constitute one feasible trajectory with continuous path acceleration. Finally, a smoothened time-optimal trajectory under the maximum velocity curve is generated by executing the bidirectional integration operation at each intersection point.
On this basis, a novel tradability mechanism is proposed as shown in Fig. \ref{fig_intro}. In order to adjust the traveling time and cruise proportion of generated trajectories under the MVC, a user-specified functional parameter is formulated as a straight line of velocity limit which reconstructs the MVC with minimum operation. As the parameter decreases, the path velocity of the reconstructed MVC decreases, while the constant velocity parts in the reconstructed MVC increase. Thus, the traveling time and cruise proportion of the generated trajectories increase. In addition, the constant velocity parts in the reconstructed MVC can be computed in $O(1)$ time complexity. Thus, feasible trajectories with a high proportion of cruise motion need less computational time. The corresponding proofs are provided in this paper. Comparative simulation and experimental results on active-caster-based omnidirectional wheeled mobile robots (OWMR) verify the effectiveness and capability  of the proposed algorithm.

The \textbf{\emph{main contributions}} of this paper are summarized as follows:

\begin{enumerate}
  \item  Under acceleration-continuous constraints, the proposed algorithm smoothens time-optimal trajectories to guarantee continuous acceleration by executing the bidirectional integration operation instead of polynomial interpolation. Accordingly, an appealing planning property `completeness' is preserved in the proposed algorithm.
  \item  A novel built-in tradability mechanism between cruise motion and time-optimal motion is proposed. As the value of a user-specified parameter increases, a faster smoothened trajectory is obtained to improve motion efficiency of robotic systems. As the value of this parameter decreases, a slower but feasible trajectory possessing a higher proportion of cruise motion is obtained to make robot move steadily and improve tracking accuracy.
  \item  The computational time of the proposed algorithm possesses upper and lower limits. The computational time of generating smoothened time-optimal trajectories reaches the upper limit which is the same with the online technique \cite{2017_Mech}, while the generation of feasible trajectories with a high proportion of cruise motion requires less computational time. It conforms to the intuition: feasible trajectories are generated faster than optimal trajectories.
\end{enumerate}

The remainder of this paper is divided into five sections. Section II summarizes related path-constrained trajectory planning algorithms. Section III describes the detailed procedures of the proposed algorithm. Section IV provides mathematical proofs for appealing properties of the proposed algorithm. In Section V, comparative simulation and experimental results on omnidirectional wheeled mobile robots demonstrate the validity of the proposed algorithm. Finally, Section VI gives some conclusions.

\section{Related Works}


\subsection{Time-Optimal Trajectory Planning}

Time-optimal trajectory planning attracts significant research attention because they increase production efficiency of robotic systems and bring great commercial profit. Based on the Pontryagin Maximum Principle, the proposed methods in \cite{1985_TAC,1985_IJRR} generate time-optimal trajectories with bang-bang torque inputs of industrial manipulators. The work in \cite{1999_TRA} applies the method \cite{1985_IJRR} to output time-optimal trajectories under both acceleration and velocity constraints for wheeled mobile robots. A fast and open source code implementation of the method \cite{1985_TAC} is provided in \cite{2014_TRO}. Recently, under both torque and velocity constraints, the work in \cite{2017_Mech} proposes a provably complete and time-optimal path-constrained trajectory planning algorithm \cite{2017_Mech}. Alternative time-optimal trajectory planning methods are proposed with dynamic programming techniques in \cite{1986_TAC,1987_JDSMC}. Convex optimization techniques are also applied to the generation of time-optimal trajectories \cite{2009_TAC}. Although the reference trajectories generated by aforementioned methods are time-optimal, the input of robotic systems is a bang-bang structure and the path acceleration along the given path is discontinuous. It results in two undesired effects. First, the planned trajectory with discontinuous acceleration is difficult to be followed, and the planned trajectory requires the saturation of at least one actuator in the motion, so that the tracking accuracy is greatly reduced especially in the presence of external uncertainties such as friction. Second, following unsmooth reference trajectories may cause safety risks and decrease the life of actuators.

\subsection{Smooth Trajectory Planning}

In order to guarantee the continuous path acceleration, most works use piecewise polynomial interpolation to express feasible trajectories \cite{1988_ICRA}, and then they compute the optimal solutions with existing sequential quadratic programming (SQP) \cite{2006_Mech,2013_RCIM}, flexible tolerance method (FTM) \cite{2000_JRS}, particle swarm optimization (PSO) \cite{2017_RCIM} and active-set \cite{2017_ICRA} optimization solvers. With a preset traveling time, the works in \cite{2000_TIE,1997_ICRA} describe smooth trajectories as cubic splines and minimize the time integral of the squared jerk along the given path. The feasibility issues in the planning problem have been discussed in \cite{2013_TRO,2007_ICMA,2005_IROS,2006_IFAC,2015_CASE}. However, these trajectory planning methods need to set the predefined traveling time, which limits their application scopes. In addition, with cubic splines and quintic polynomials, the works in \cite{2003_TRA,2000_JRS} minimize an objective function of the traveling time with enforced constraints of bounded jerk to guarantee smoothness. The work in \cite{2017_RCIM} first computes a time-optimal trajectory expressed by piecewise cubic splines, and then fixes knots among the cubic splines with the 7th order polynomial. In the works \cite{2006_Mech,2011_JIRS}, the objective function consists of two terms: the traveling time and integral of the square jerk along the given path. By adjusting the weight of two terms, smoother or faster trajectories can be generated. To summarize, for aforementioned planning techniques, there are still some unsolved issues as follows:

\begin{enumerate}
  \item The time optimality cannot be guaranteed in global space due to piecewise polynomial interpolation.
  \item The computational time of trajectory generation is not adjustable, and the real-time performance is usually difficult to be guaranteed due to the nonlinear optimization.
  \item An important algorithmic property `completeness' is ignored, which indicates that a planning algorithm returns feasible solutions for solvable planning problems otherwise a failure is returned in finite time. The detailed statements of the algorithmic completeness is also discussed in our previous works \cite{2017_Mech}.
\end{enumerate}

From the above analysis, there is still a lack of a real-time acceleration-continuous path-constrained trajectory planning, which possesses an effective tradability mechanism between cruise motion and time-optimal motion. It will be provided in this paper, and the corresponding appealing properties and proofs are also given in following sections.

\section{The Proposed Algorithm}

This section provides the procedures of the proposed algorithm which generates smoothened time-optimal trajectories, as well as feasible trajectories possessing a high proportion of cruise motion.

\subsection{Path Parameterization}

Along a given path, the trajectory planning problem is transformed from a high-dimensional space into a two-dimensional space consisting of path coordinate and path velocity. For instance, the dynamically extended model of a first-order nonlinear system (robot manipulators \cite{2016_ACC}, unicycle robots \cite{2016_TCST}, car-like robots \cite{2001_TRA}, tractor-trailer robots \cite{2014_Mech}, OWMR \cite{2015_Cyber}) is described as follows:
\begin{align}
\label{eq_v}
\bm v &= J(\bm q)\bm{\dot q},  \\
\label{eq_a}
\bm a & = \bm{\dot v},
\end{align}
where the vector $\bm q\in \mathbb{R}^n$ is the state of the robotic system, and the vectors $\bm v, \bm a\in\mathbb{R}^m$ are the velocity and acceleration of actuators, respectively. The matrix $J(\bm q)\in \mathbb{R}^{m\times n}$ is a function of $\bm q$ denoting the Jacobian matrix.

Along a specified path, the state of the robotic system is represented as $\bm q(s)$ with the scalar $s$ being path coordinate. Accordingly, the equations (\ref{eq_v}) and (\ref{eq_a}) are rewritten as
\begin{align}
\label{eq_v_s}
\bm v &= J(\bm q(s))\bm{q}_s \dot s,  \\
\label{eq_a_s}
\bm a & = J(\bm q(s))\bm q_s\ddot s + (J_s\bm q_s+J(\bm q(s))\bm q_{ss})\dot{s}^2,
\end{align}
wherein $J_s = \partial J(\bm q(s))/\partial s$, $\bm q_s=\partial \bm q(s)/\partial s$ and $\bm q_{ss} = \partial^2 \bm q(s) / \partial s^2$. The scalars $\dot s$ and $\ddot s$ are the path velocity and path acceleration, respectively.

The following inequalities represent velocity and acceleration constraints of actuators of the robotic system,
\begin{align}
\label{eq_v_cons}
-\bm v_{max} \leq \bm v \leq \bm v_{max},  \\
\label{eq_a_cons}
-\bm a_{max} \leq \bm a \leq \bm a_{max},
\end{align}
where $\bm v_{max}\in\mathbb{R}^m$ and $\bm a_{max}\in\mathbb{R}^m$ are constant vectors representing velocity and acceleration bounds, respectively.

In order to guarantee velocity constraints (\ref{eq_v_cons}), substituting (\ref{eq_v_s}) into (\ref{eq_v_cons}) yields that
\begin{align}\label{eq_ad}
\bm A(s) \dot s + \bm D(s) \leq \bm 0,
\end{align}
with
\begin{align}
\label{eq_A}
\bm A(s) &= [(J(\bm q(s))\bm{q}_s)^\mathrm{T} \ -(J(\bm q(s))\bm{q}_s)^\mathrm{T}]^\mathrm{T}, \\
\label{eq_D}
\bm D(s) &= [-\bm v^\mathrm{T}_{max} \ -\bm v^\mathrm{T}_{max}]^\mathrm{T}.
\end{align}

In order to guarantee acceleration constraints (\ref{eq_a_cons}), substituting (\ref{eq_a_s}) into (\ref{eq_a_cons}) yields that
\begin{align}\label{eq_abc}
\bm A(s) \ddot s + \bm B(s) \dot{s}^2 + \bm C(s) \leq \bm 0,
\end{align}
with
\begin{align}
\label{eq_B}
\bm B(s) &= [(J_s\bm q_s+J(\bm q(s))\bm q_{ss})^\mathrm{T} \ -(J_s\bm q_s+J(\bm q(s))\bm q_{ss})^\mathrm{T}]^\mathrm{T}, \\
\label{eq_C}
\bm C(s) &= [-\bm a^\mathrm{T}_{max} \ -\bm a^\mathrm{T}_{max}]^\mathrm{T}.
\end{align}

With the aid of the inequality (\ref{eq_abc}), path acceleration constraints are computed as
\begin{align}
\label{eq_alp_bet}
\alpha(s,\dot s) \leq \ddot s \leq \beta(s,\dot s),
\end{align}
where lower and upper bounds are described as
\begin{align}
\alpha(s,\dot s) &= \mathrm{max}\left\{\alpha_i|\alpha_i = \dfrac{-B_i(s)\dot s^2 - C_i(s)}{A_i(s)}, A_i(s)<0 \right\}, \label{eq_alp} \\
\beta(s,\dot s) &= \mathrm{min}\left\{\beta_i|\beta_i = \dfrac{-B_i(s)\dot s^2 - C_i(s)}{A_i(s)}, A_i(s)>0 \right\}, \label{eq_bet}
\end{align}
with the scalars $A_i(s), B_i(s), C_i(s)$ are elements of vectors $\bm A(s),\bm B(s),\bm C(s)$, respectively.

From above analysis, the model and constraints of the robotic system are transformed from a $n$-dimensional state space into a two-dimensional $(s,\dot s)$ space.

\begin{figure}[t]
      \centering
      \includegraphics[scale=0.7]{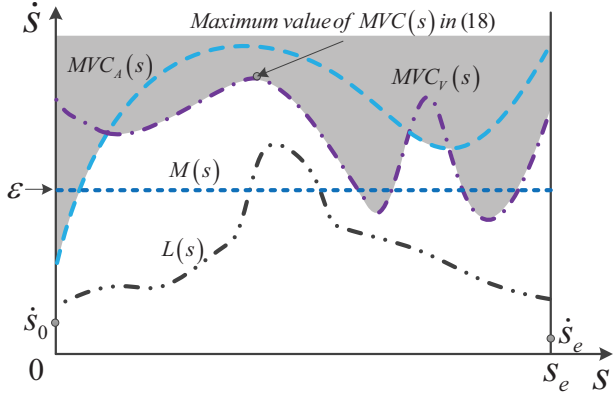}
      \caption{It is a diagram of velocity limit curves. The cyan dash curve represents the velocity limit curve $MVC_A(s)$ satisfying acceleration constraints (\ref{eq_a_cons}). The purple dash-dot curve represents the velocity limit curve $MVC_V(s)$ satisfying velocity constraints (\ref{eq_v_cons}). The blue small dash line stands for the velocity limit curve decided by $\varepsilon$. The black dash-dot-dot curve represents a constant velocity boundary on the plane $(s,\dot s)$. The symbols $\dot s_0, \dot s_e, s_e$ are the starting path velocity, terminal path velocity and total length of a specified path.}
      \label{fig_limit_curve}
\end{figure}

\subsection{Velocity Limit Curve}

In this subsection, velocity and acceleration constraints of robotic systems are represented as velocity limit curves on the phase plane $(s,\dot s)$. A user-specified functional parameter is also introduced as a user-prone constant velocity limit curve to adjust the traveling time, cruise proportion and computational time of feasible trajectories.

According to (\ref{eq_ad}), a velocity limit curve satisfying velocity constraints (\ref{eq_v_cons}) is computed as
\begin{equation} \label{eq_mvc_v}
MVC_V(s)=\mathrm{min}\left\{\dot s \geq 0 | -\dfrac{D_i(s)}{A_i(s)}, A_i(s) > 0 \right\},   s\in[0,s_e],
\end{equation}
where the scalar $s_e$ is the total length of the given path and $D_i(s)$ is the element of $\bm D(s)$. For example, $MVC_V(s)$ is represented as the dash-dot curve in Fig. \ref{fig_limit_curve}. When the path velocity of the robotic system reaches $MVC_V(s)$, there exists at least one saturated actuator velocity.

With the aid of the equation (\ref{eq_alp_bet}), the velocity limit curve satisfying acceleration constraints (\ref{eq_a_cons}) is computed as
\begin{equation} \label{eq_mvc_a}
MVC_A(s)=\mathrm{min}\{\dot s \geq 0 \mid \alpha(s,\dot s) = \beta(s, \dot s) \}, \  s\in[0,s_e].
\end{equation}
In Fig. \ref{fig_limit_curve}, $MVC_A(s)$ is represented as a dash cyan curve, where the equation $\alpha(s,\dot s)=\beta(s,\dot s)$ holds. When the path velocity of the robotic system  reaches $MVC_A(s)$, there exists at least one actuator achieving acceleration saturation.

In order to satisfy velocity and acceleration constraints, the maximum velocity curve is obtained as
\begin{equation} \label{eq_mvc_a}
MVC(s)=\mathrm{min}(MVC_A(s), MVC_V(s)), \  s\in[0,s_e].
\end{equation}
It is represented as a boundary between gray and non-gray regions in Fig. \ref{fig_limit_curve}. Constraints (\ref{eq_v_cons}) and (\ref{eq_a_cons}) are satisfied in the non-gray region.

With the aid of complete numerical integration techniques \cite{2017_Mech}, a time-optimal trajectory is obtained under the maximum velocity curve. It consists of accelerating velocity profiles $\beta(s,\dot s)$, decelerating velocity profiles $\alpha(s,\dot s)$ and switch arcs, where switch arcs are feasible parts satisfying path acceleration constraints (\ref{eq_alp_bet}) in the maximum velocity curve \cite{2017_Mech}. Thus, the feasible trajectory is attached to the maximum velocity curve. Note that the path acceleration of the feasible trajectory is discontinuous. For instance, accelerating velocity profiles $(\beta_1,\beta_2)$, decelerating velocity profiles $(\alpha_1,\alpha_2)$ and the switch arc $\overline{p_1p_2}$ constitute the feasible trajectory attached to the maximum velocity curve in Fig. \ref{fig_traj_generation}. Its discontinuous path acceleration occurs at points $p_3$ and $p_4$, which will be smoothened to achieve acceleration-continuous trajectory in Section III-C.

\subsection{Tradability Mechanism}

In order to adjust the traveling time and cruise proportion of the feasible trajectories attached to the maximum velocity curve, a functional parameter $\varepsilon$ changing the maximum velocity curve is introduced as an user-adjustable path velocity constraint
\begin{equation}\label{eq_param_cons}
\dot s \leq \varepsilon, \ s\in[0,s_e],
\end{equation}
where $\varepsilon \in [\mathrm{max}(\dot s_0,\dot s_e), \mathrm{Max}(MVC)]$ with $\mathrm{Max}(MVC)$ being the maximum value of $MVC(s)$ with $s\in[0,s_e]$.

In order to satisfy the constraint (\ref{eq_param_cons}), a velocity limit curve is described as
\begin{equation}\label{eq_m_param}
M(s) = \varepsilon, \ s\in[0,s_e].
\end{equation}
In Fig. \ref{fig_limit_curve}, $M(s)$ is represented as a blue small dash straight line, which moves up and down between  $\dot s_0$ and $\mathrm{Max}(MVC)$.

The velocity limit curve decided by $\varepsilon$ reconstructs the maximum velocity curve as
\begin{equation} \label{eq_mvc_star}
MVC^*(s)=\mathrm{min}(MVC(s), M(s)), \  s\in[0,s_e].
\end{equation}

\begin{figure}[t]
      \centering
      \includegraphics[scale=0.7]{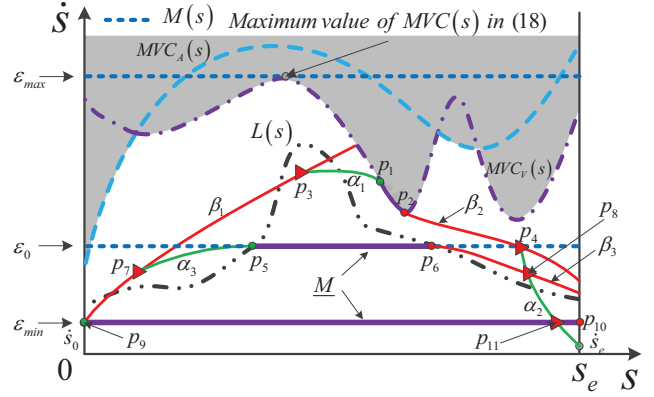}
      \caption{It is a diagram of the adjustment of the functional parameter $\varepsilon$. The amplitude and shape of $MVC^*(s)$ change with the variation of $\varepsilon$ $(\varepsilon_{max}\mapsto\varepsilon_{0}\mapsto\varepsilon_{min})$. The decreasing of $\varepsilon$ indicates that the traveling time and cruise proportion of feasible trajectories attached to $MVC^*(s)$ increase. The purple solid straight lines represent $\underline{M}$, which can be taken as switch arcs in $MVC^*(s)$. The green $\lhd$ stands for the switch point in $MVC^*(s)$. The red $\rhd$ represents the intersection point among switch arcs ($\overline{p_5p_6},\overline{p_{9}p_{10}}$), accelerating velocity profiles ($\beta_1,\beta_2,\beta_3$) and decelerating velocity profiles ($\alpha_1,\alpha_2,\alpha_3$).  }
      \label{fig_traj_generation}
\end{figure}

The adjustment of the functional parameter affects the amplitude and shape of $MVC^*(s)$. The physical interpretation of $\varepsilon$ is given as follows:
\begin{itemize}
  \item $\varepsilon \rightarrow \mathrm{max}(\dot s_0,\dot s_e)$: The amplitude of $MVC^*(s)$ decreases, and it indicates that the traveling time of generated trajectories corresponding to $MVC^*(s)$ increases. Simultaneously, the shape of $MVC^*(s)$ tends to be a straight line, and it indicates that feasible trajectories possess a higher proportion of cruise motion.
  \item $\varepsilon \rightarrow \mathrm{Max}(MVC)$: The amplitude of $MVC^*(s)$ increases, and it indicates that the traveling time of generated trajectories corresponding to $MVC^*(s)$ decreases and tends to be optimal. Simultaneously, the shape of $MVC^*(s)$ is closer to $MVC(s)$, and it indicates that the cruise proportion of feasible trajectories decreases.
\end{itemize}
For example, the functional parameter $\varepsilon$ is adjusted to $\varepsilon_{max}$, $\varepsilon_{0}$ and $\varepsilon_{min}$ in Fig. \ref{fig_traj_generation}. For $\varepsilon = \varepsilon_{max}$, the generated trajectory consists of the switch arc $\overline{p_1p_2}$, accelerating ($\beta_1,\beta_2$) and decelerating ($\alpha_1,\alpha_2$) velocity profiles. For $\varepsilon=\varepsilon_{0}$, the generated trajectory consists of the switch arc $\overline{p_5p_6}$, accelerating ($\beta_1,\beta_3$) and decelerating $(\alpha_2,\alpha_3)$ velocity profiles. For $\varepsilon=\varepsilon_{min}$, the generated trajectory consists of the switch arc $\overline{p_{9}p_{10}}$ and decelerating $(\alpha_2)$ velocity profile. As $\varepsilon$ reduces $(\varepsilon_{max}\mapsto\varepsilon_{0}\mapsto\varepsilon_{min})$, the traveling time and cruise proportion of feasible trajectories increase. Note that the discontinuous acceleration of the feasible trajectories at red $\rhd$ will be addressed with a bidirectional integration operation in Section III-C.

The generation of feasible trajectories attached to $MVC^*(s)$ needs to search switch arcs along $MVC^*(s)$ \cite{2017_Mech}. The computational time of searching switch arcs is in proportional to the length of the given path. Therefore, in order to reduce the computational time, we introduce a new concept called `constant velocity boundary' in order to speed up the computation of switch arcs in $M(s)$ of $MVC^*(s)$.

\begin{defn} \label{def_cvb}
Constant velocity boundary is a continuous curve which divides the plane $(s,\dot s)$ into constant and non-constant velocity regions. Below the boundary, constant velocity profiles satisfy the path acceleration constraints, while above the boundary, constant velocity profiles violate the path acceleration constraints. 
\end{defn}

According to the \emph{Definition \ref{def_cvb}}, the computational formula of the constant velocity boundary is given as
\begin{equation} \label{eq_cvb}
L(s)\!=\!\mathrm{min}\left\{ \sqrt{-\dfrac{C_i(s)}{B_i(s)}} | \dfrac{C_i(s)}{B_i(s)}\leq0, A_i(s) \!\neq 0\! \right\},  s\in[0,s_e].
\end{equation}
For instance, $L(s)$ is represented as the black dash-dot-dot curve in Fig. \ref{fig_limit_curve}. In the implementation code, the curve $L(s)$ is stored in a table structure. By querying the table, the solution of the equation $L(s)=M(s)$ is solved in $O(1)$ time complexity to obtain
\begin{align}
\overline{M} &= \{M(s)|M(s)>L(s), s\in[0,s_e]\}, \label{eq_over_M} \\
\underline{M}&= \{M(s)|M(s)\leq L(s), s\in[0,s_e]\}. \label{eq_under_M}
\end{align}
Both lines $\overline{M}$ and $\underline{M}$ possess the same path acceleration $\ddot s=0$, however only $\underline{M}$ satisfies path acceleration constraints (\ref{eq_alp_bet}). With respect to $O(s_e)$ time complexity of searching along $MVC^*(s)$ \cite{2017_Mech}, the look-up table query only has the time complexity of $O(1)$ to obtain $\underline{M}$ as switch arcs of feasible trajectories, such as the purple straight solid lines in Fig. \ref{fig_traj_generation}. Thus, the constant velocity boundary reduces the computational time of switch arcs in $M(s)$ of $MVC^*(s)$.

\begin{rem}
According to the equation (\ref{eq_cvb}), the table structure computation of $L(s)$ is derived from the equation (\ref{eq_abc}), which is same with the computation of velocity limit curves $MVC_A(s)$, $MVC_V(s)$ and $M(s)$. Therefore, the constant velocity boundary $L(s)$ together with these velocity limit curves can be obtained through parallel computation. \QEDA
\end{rem}

\subsection{Trajectory Generation}

This subsection describes the procedures of the proposed algorithm generating feasible trajectories under those velocity limit curves as follows. First, with the aid of complete numerical integration techniques \cite{2017_Mech}, time-optimal trajectories with discontinuous path acceleration are generated under $MVC^*(s)$. Then, the discontinuous path acceleration is addressed by a bidirectional integration operation. Through adjusting the user-specified functional parameter $\varepsilon$ in (\ref{eq_m_param}), the straight line $M(s)$ moving up and down reconstructs $MVC^*(s)$ to change the computational time, traveling time and cruise proportion of generated trajectories in real time.

\textbf{Complete numerical integration (CNI)} computes time-optimal trajectories with discontinuous path acceleration under $MVC^*(s)$:

\begin{figure}[t]
      \centering
      \includegraphics[scale=0.7]{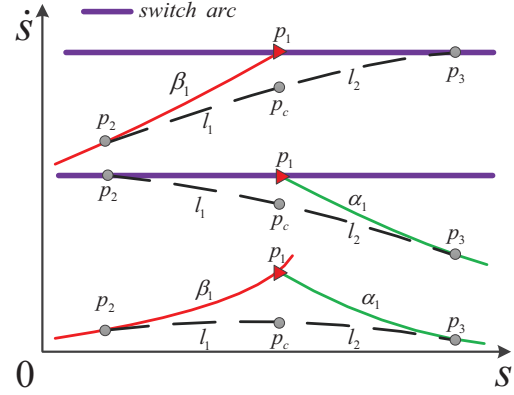}
      \caption{It is a diagram of three intersection cases among switch arcs, accelerating and decelerating velocity profiles. Top: The accelerating velocity profile $\beta_1$ intersects the switch arc at the point $p_1$. Middle: The switch arc intersects the decelerating velocity profile $\alpha_1$ at the point $p_1$. Bottom: The accelerating velocity profile $\beta_1$ intersects the decelerating velocity profile $\alpha_1$ at the point $p_1$. The black big dash curves are computed by the proposed bidirectional integration operation. }
      \label{fig_discontinuous_acceleration}
\end{figure}

1) Switch search: The goal of this step is searching switch arcs and switch points along $MVC^*(s)$. There are three types of switch points in $MVC^*(s)$: Tangent points, discontinuity points and zero-inertia points \cite{2017_Mech}. The switch points and endpoints of switch arcs are used as the starting points of numerical integration: Forward integration and Backward integration.

2) Forward integration: The goal is computing accelerating velocity profiles by doing forward numerical integration from the point $(0,\dot s_0)$,  switch points and right endpoints of switch arcs with maximum path acceleration $\beta(s,\dot s)$.

3) Backward integration: The goal is computing decelerating velocity profiles by doing backward numerical integration from the point $(s_e,\dot s_e)$, switch points and left endpoints of switch arcs with minimum path acceleration $\alpha(s,\dot s)$.

Time-optimal trajectories under $MVC^*(s)$ are generated by intersecting these accelerating and decelerating velocity profiles and switch arcs. For example, a time-optimal trajectory consists of switch arc ($\overline{p_5p_6}$), accelerating velocity profiles ($\beta_1,\beta_3$) and decelerating velocity profiles ($\alpha_2,\alpha_3$) for $\varepsilon=\varepsilon_0$ as shown in Fig. \ref{fig_traj_generation}. The path acceleration of the time-optimal trajectory is discontinuous at the intersection points, such as $p_7$ and $p_8$ in Fig. \ref{fig_traj_generation}. There are three intersection cases among switch arcs and accelerating and decelerating velocity profiles as shown in Fig. \ref{fig_discontinuous_acceleration}. In order to guarantee continuous path acceleration, a bidirectional integration operation is proposed around the intersection points.

\textbf{Bidirectional integration operation (BIO)} computes acceleration-continuous velocity profiles to connect trajectories at two sides of intersection points:

1) At two sides of the intersection point $p_1$, find two arbitrary points $p_2$ and $p_3$ in the generated trajectory by CNI, so that the point $p_1$ is the only intersection point between $p_2$ and $p_3$, as shown in Fig. \ref{fig_discontinuous_acceleration}. The scalars $s_{p_i},\dot s_{p_i}$ represent the path coordinate and path velocity of the point $p_i, i\in[1,3]$ on the plane $(s,\dot s)$, while the scalars $\ddot s_{p_2}$ and $\ddot s_{p_3}$ represent the path acceleration of the generated trajectories by CNI at $p_2$ and $p_3$, respectively.

2) Starting from $p_2$, one velocity profile $l_1$ is computed by doing forward numerical integration using the path acceleration $\ddot s_1$ as follow:
\begin{equation}\label{eq_forward_acc}
\ddot s_1=\beta^*(s,\dot s)-\dfrac{(s-s_{p_2})(\beta^*(s,\dot s)-\alpha^*(s,\dot s))}{s_{p_3}-s_{p_2}}
\end{equation}
until $s=s_{p_1}$.
Starting from $p_3$, another velocity profile $l_2$ is computed by doing backward numerical integration using the path acceleration $\ddot s_2$ as follow:
\begin{equation}\label{eq_backward_acc}
\ddot s_2=\alpha^*(s,\dot s)+\dfrac{(s_{p_3}-s)(\beta^*(s,\dot s)-\alpha^*(s,\dot s))}{s_{p_3}-s_{p_2}}
\end{equation}
until $s=s_{p_1}$. In (\ref{eq_forward_acc}) and (\ref{eq_backward_acc}), the $\beta^*(s,\dot s)$ and $\alpha^*(s,\dot s)$ are computed as
\begin{align}
\beta^*(s,\dot s) &= \alpha(s,\dot s) + (\beta(s,\dot s)-\alpha(s,\dot s))\delta_1, \label{eq_deta_1} \\
\alpha^*(s,\dot s) &= \alpha(s,\dot s) + (\beta(s,\dot s)-\alpha(s,\dot s))\delta_2, \label{eq_deta_2}
\end{align}
with
\begin{align}
\delta_1 &= (\ddot s_{p_2}-\alpha(s_{p_2},\dot s_{p_2}))/(\beta(s_{p_2},\dot s_{p_2})-\alpha(s_{p_2},\dot s_{p_2})), \nonumber \\
\delta_2 &= (\ddot s_{p_3}-\alpha(s_{p_3},\dot s_{p_3}))/(\beta(s_{p_3},\dot s_{p_3})-\alpha(s_{p_3},\dot s_{p_3})). \nonumber
\end{align}

\begin{thm} \label{thm_1}
The velocity profile consisting of $l_1$ and $l_2$ is continuous, and its path acceleration is also continuous at $s_{p_1}$.
\end{thm}

\begin{proof}
When $p_1$ is chosen as $p_2$ and $p_3$, it holds that $p_1$ is the connection point $p_c$ of $l_1$ and $l_2$. According to the continuity of $\alpha(s,\dot s)$ and $\beta(s,\dot s)$ on the plane $(s,\dot s)$ \cite{2017_RAL}, the connection point $p_c$ of $l_1$ and $l_2$ still exists when two arbitrary points at two sides of $p_1$ are chosen as $p_2$ and $p_3$ as shown in Fig. \ref{fig_discontinuous_acceleration}. Therefore, the velocity profile consisting of $l_1$ and $l_2$ is continuous.

The path acceleration $\ddot s_1$ and $\ddot s_2$ of $l_1$ and $l_2$ at $s_{p_1}$ are obtained by substituting the path coordinate and path velocity of $p_c$ into (\ref{eq_forward_acc}) and (\ref{eq_backward_acc}), respectively. According to (\ref{eq_forward_acc}) and (\ref{eq_backward_acc}), the difference between $\ddot s_1$ and $\ddot s_2$ is equal to zero. Thus, the path acceleration of the velocity profile consisting of $l_1$ and $l_2$ is continuous at $s_{p_1}$. \QEDA
\end{proof}

The whole procedure of the proposed algorithm is given in \textbf{Algorithm \ref{alg_1}}. According to (\ref{eq_mvc_star}) and (\ref{eq_under_M}), the subfunction does parallel computation to obtain $MVC^*$ decided by the user-specified parameter $\epsilon$ and switch arcs $\underline{M}$ in $MVC^*$. Then, with the aid of $MVC^*$ and $\underline{M}$, the subfunction CNI outputs a result $\Omega$ of constructing trajectory between $\dot s_0$ and $\dot s_e$. The $\Omega$ may be an acceleration-discontinuous trajectory or failure \cite{2017_Mech}. For the acceleration-discontinuous $\Omega$, the subfunction BIO is called to return a smoothened time-optimal trajectory.

\begin{algorithm}[t]
\caption{The proposed algorithm}\label{alg_1} 
\hspace*{0.02in} {\bf Input:} 
${MVC(s)}$,~${L(s)}$,~${\dot s_0}$,~${\dot s_e}$,~${\epsilon}$.\\
\hspace*{0.02in} {\bf Output:} 
a solution or failure.
\begin{algorithmic}[1]
\STATE {$MVC^*, \underline{M} \leftarrow$ PreCompute$(MVC(s),L(s),\epsilon)$}
\STATE {$\Omega \leftarrow$ CNI$(MVC^*, \underline{M}, \dot s_0, \dot s_e)$ }
\IF {$\Omega$ is a failure }
    \RETURN { failure }
\ENDIF
\STATE {solution $\leftarrow$ BIO$(\Omega)$}
\RETURN {solution}
\end{algorithmic}
\end{algorithm}

\begin{thm} \label{thm_1_1}
The proposed algorithm is complete for a path-constrained planning problem.
\end{thm}

\begin{proof}
The proposed algorithm consists of CNI and BIO. The work in \cite{2017_Mech} has proven that CNI is complete for a path-constrained planning problem. When the planning problem is unsolvable, the CNI returns a failure in finite time, and then the proposed algorithm terminates and returns the failure immediately. When the planning problem is solvable, the CNI outputs a feasible trajectory with discontinuous path acceleration, which occurs at intersection points among switch arcs, accelerating and decelerating velocity profiles. Then, BIO is called to compute acceleration-continuous feasible velocity profiles to connect trajectories at two sides of intersection points, which haven been proven in \emph{Theorem \ref{thm_1}}. The number of intersection points is finite \cite{1985_TAC}, thus the proposed algorithm does a finite number of BIO function invocations and it returns an acceleration-continuous feasible trajectory in finite time.  \QEDA
\end{proof}



\section{Proof of Tradability Mechanism}

This section provides three theorems of the tradability mechanism of the proposed algorithm with detailed proofs.

\begin{thm} \label{thm_2}
With the increasing of $\varepsilon$, the traveling time of trajectories generated by the proposed algorithm decreases until it achieves the time-optimal motion.
\end{thm}

\begin{proof}
According to (\ref{eq_mvc_star}), the adjustment of $\varepsilon$ from $\varepsilon_2$ to $\varepsilon_1$ ($\varepsilon_2 < \varepsilon_1$) indicates that $MVC^*(s)$ with $\varepsilon_1$ is greater than or equal to $MVC^*(s)$ with $\varepsilon_2$ at each path coordinate $s\in[0,s_e]$. Namely, the admissible region under $MVC^*(s)$ with $\varepsilon_1$ contains the admissible region under $MVC^*(s)$ with $\varepsilon_2$. With the aid of CNI, the trajectory generated by the proposed algorithm is time-optimal under $MVC^*(s)$. Thus, the traveling time of  the trajectory generated under $MVC^*(s)$ with $\varepsilon_1$ is less than the traveling time of the trajectory generated under $MVC^*(s)$ with $\varepsilon_2$.  \QEDA
\end{proof}

\begin{figure}[t]
      \centering
      \includegraphics[scale=0.6]{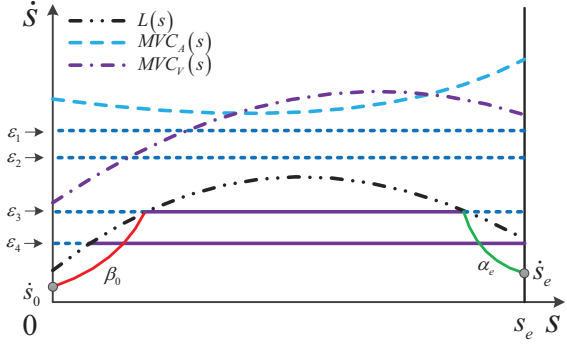}
      \caption{The adjustment of the functional parameter $\varepsilon$ changes the ratio of $\underline{M}, \overline{M}, MVC(s)$ in $MVC^*(s)$. The purple solid straight lines represent $\underline{M}$. The blue small dash straight lines represent $M(s)$. The red and green solid curves represent accelerating and decelerating velocity profiles, respectively. When the functional parameter $\varepsilon$ increases from $\varepsilon_4$ to $\varepsilon_1$, the ratio of $\underline{M}$ in $MVC^*(s)$ decreases, and the ratio of $MVC(s)$ in $MVC^*(s)$ increases.}
      \label{fig_proof}
\end{figure}

\begin{thm} \label{thm_3}
With the increasing of $\varepsilon$, the computational time of trajectories generated by the proposed algorithm increases.
\end{thm}

\begin{proof}
Under the curve $MVC^*(s)$, the resultant trajectory generated by the proposed algorithm consists of switch arcs, accelerating and decelerating velocity profiles assigned to the given path with the total path length $s_e$. According to (\ref{eq_mvc_star}), switch arcs in $MVC^*(s)$ include the straight lines $\underline{M}$ in $M(s)$ and the curves satisfying path acceleration constraints (\ref{eq_alp_bet}) in $MVC(s)$. With the aid of $L(s)$ in (\ref{eq_cvb}), the $\underline{M}$ is computed as switch arcs with $O(1)$ time complexity. According to (\ref{eq_under_M}), the increasing of $\varepsilon$ indicates that the ratio of $\underline{M}$ in $MVC^*(s)$ decreases. Namely, the corresponding path length $s^\dag$ of $\underline{M}$ decreases as shown in Fig. \ref{fig_proof}. Then, the corresponding path length $s_e-s^\dag$ of the switch arcs in $MVC(s)$, accelerating and decelerating velocity profiles increases, as shown in the case when $\varepsilon$ is adjusted from $\varepsilon_4$ to $\varepsilon_3$ in Fig. \ref{fig_proof}.  The work in \cite{2017_Mech} has described that the time complexity of computing the switch arcs in $MVC(s)$, accelerating and decelerating velocity profiles is $O(s_e-s^\dag)$. Thus, with the increasing of $\varepsilon$, the computational time of resultant trajectories increases. \QEDA
\end{proof}

\begin{figure}[t]
      \centering
      \includegraphics[scale=0.6]{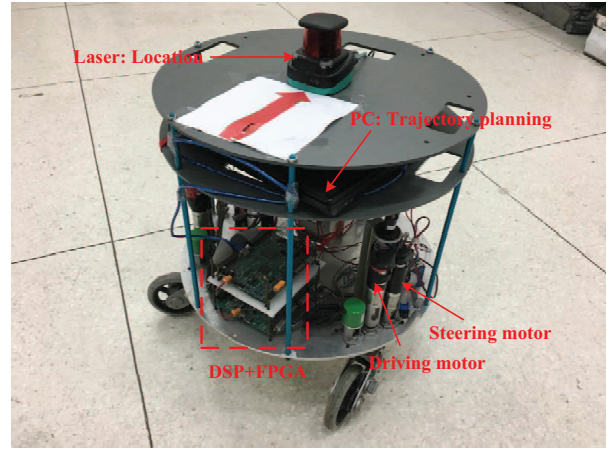}
      \caption{NK-OMNI I \cite{2015_Cyber}: The omnidirectional wheeled mobile robot uses two active casters and one passive caster.}
      \label{fig_nkomni}
\end{figure}

\begin{thm} \label{thm_4}
With the increasing of $\varepsilon$, the cruise proportion of trajectories generated by the proposed algorithm decreases.
\end{thm}

\begin{proof}
The adjustment of the functional parameter $\varepsilon$ changes the ratio of $\underline{M}$ in $MVC^*(s)$. According to (\ref{eq_under_M}), when the functional parameter $\varepsilon$ increases, the ratio of $\underline{M}$ in $MVC^*(s)$ decreases, as shown in Fig. \ref{fig_proof}. In the `switch search' stage of CNI of the proposed algorithm, the $\underline{M}$ in $MVC^*(s)$ is chosen as switch arcs which constitute constant velocity parts of the resultant trajectory. Therefore, the cruise proportion of trajectories generated by the proposed algorithm is decreasing with the increasing of $\varepsilon$. \QEDA
\end{proof}

\section{Simulation and Experimental Results}

In order to verify the proposed algorithm and related properties, this section provides simulation and experimental results compared with the methods \cite{2000_JRS,2011_JIRS} on a self-developed OWMR platform ``NK-OMNI I'' as shown in Fig. \ref{fig_nkomni}. The OWMR possesses two active casters and one passive caster. Each active caster has two motors which take charge of driving and steering motion of the caster independently. The velocity $\bm v_{max}$ and acceleration $\bm a_{max}$ constraints of active casters are listed in simulation and experimental cases. The detailed kinematic model and path parameterization of the OWMR along a given path can be found in \cite{2017_Mech}, therefore we omit these contents of model and path parameterization. In following simulation and experimental cases, a cubic B\`{e}zier curve is chosen as the given path:
\begin{align}
\bm P(\lambda) = \sum_{i=0}^3 \binom{3}{i} (1-\lambda)^{3-i} \lambda^i \bm P_i, \label{eq_bezier}
\end{align}
where $\binom{3}{i}$ is the binomial coefficient, $\lambda \in [0,1]$ is the path parameter, and $\bm P_i, i\in[0,3]$ is the path control point. The $\bm P(\lambda)$ represents the position of the center point of OWMR in a reference coordinate system. The OWMR orientation is specified as  $\pi s/s_e, s\in[0, s_e]$.

\begin{figure}[t]
      \centering
      \includegraphics[scale=0.6]{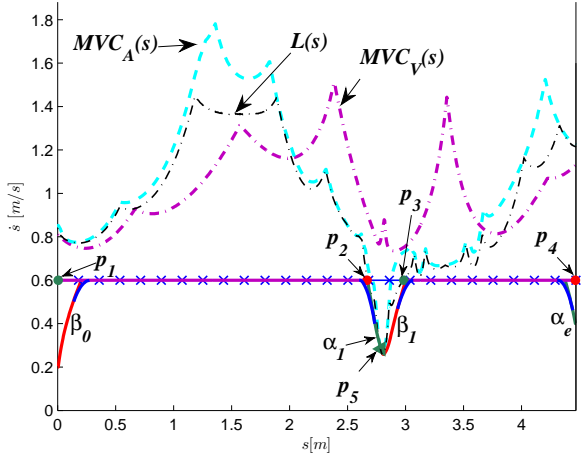}
      \caption{The proposed algorithm generates a feasible trajectory with a high proportion of cruise motion. First, the switch arc (purple solid lines $\overline{p_1p_2},\overline{p_3p_4}$), accelerating (red solid curves $\beta_0,\beta_1$) and decelerating (green solid curves $\alpha_1,\alpha_e$) velocity profiles intersect to constitute a feasible trajectory with discontinuous path acceleration. Then, the discontinuous path acceleration at intersection points is addressed by the proposed BIO to guarantee smoothness, such as those blue solid curves.}
      \label{fig_sim_traj}
\end{figure}

\subsection{Simulation Results}

In this simulation case, each element of velocity constraints $\bm v_{max}$ is set as $18[rad/s]$, and each element of acceleration constraints $\bm a_{max}$ is set as $20[rad/s^2]$. The starting and terminal path velocities are set as $\dot s_0=0.2 [m/s]$ and $\dot s_e=0.4 [m/s]$, respectively. Fig. \ref{fig_sim_traj} shows that when the functional parameter $\varepsilon$ is adjusted to 0.6, the proposed algorithm generates a feasible trajectory with a high proportion of cruise motion under $MVC^*(s)$. The cyan dash curve, purple dash-dot curve and blue solid line with star marker represent $MVC_A(s)$, $MVC_V(s)$ and $M(s)$, respectively. The black dash-dot thin curve represents $L(s)$ in (\ref{eq_cvb}), and it divides $M(s)$ into $\overline{M}$ (the blue solid line with star marker $\overline{p_2p_3}$) and $\underline{M}$ (the purple solid line $\overline{p_1p_2},\overline{p_3p_4}$). First, a feasible trajectory with discontinuous path acceleration is obtained by intersecting the switch arcs $\underline{M}$, accelerating $(\beta_0,\beta_1)$ and decelerating $(\alpha_1,\alpha_e)$ velocity profiles. Accordingly, for each intersection point, the proposed BIO is conducted to guarantee continuous path acceleration. For example, Fig. \ref{fig_sim_traj_detail} shows the enlarged view of the region around intersection points among $\underline{M}$, $\alpha_1$ and $\beta_1$. The blue solid curves derived from the BIO successfully connect $\underline{M}$, $\alpha_1$ and $\beta_1$, and ensure continuous path acceleration. It verifies \emph{Theorem \ref{thm_1}} and \emph{Theorem \ref{thm_1_1}}. For simplicity, the resultant trajectories generated by the proposed algorithm are expressed as orange solid curves (not distinguishing acceleration and deceleration) in following simulation and experimental cases.

\begin{figure}[t]
      \centering
      \includegraphics[scale=0.6]{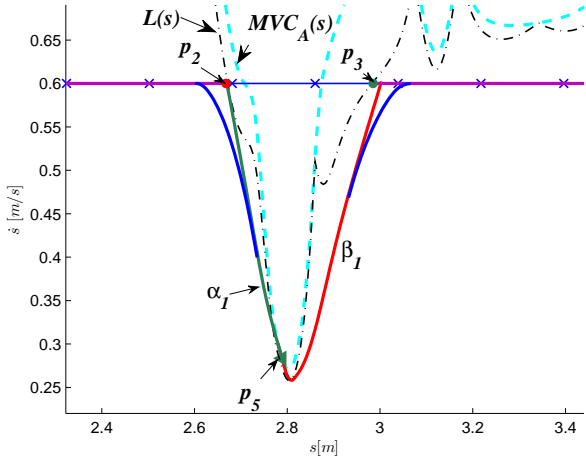}
      \caption{The enlarged view of the region around intersection points in Fig. \ref{fig_sim_traj}.}
      \label{fig_sim_traj_detail}
\end{figure}

\begin{figure}[t]
      \centering
      \includegraphics[scale=0.6]{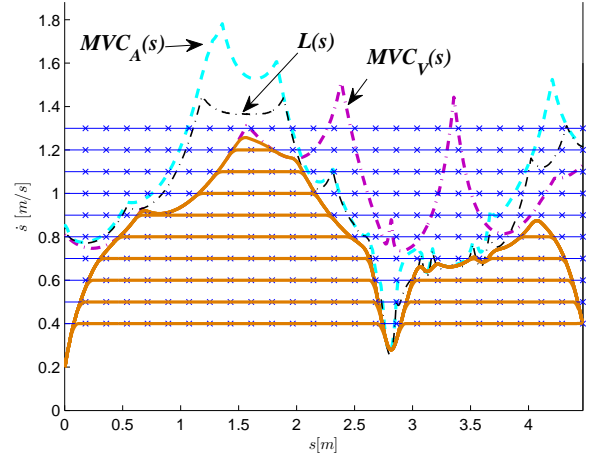}
      \caption{With the increasing of $\varepsilon$, the cruise proportion and traveling time of resultant trajectories (orange curves) generated by the proposed algorithm decrease.}
      \label{fig_sim_multitraj}
\end{figure}

\begin{table}[!htb]
  \centering
  \caption{Performance of the proposed algorithm with increasing $\varepsilon$}
  \label{tab_1}
  \begin{tabular}{|c|c|c|c|}
    \hline
    $\varepsilon$         &  Trav. Time [s]     &   Comp. Time [ms]   & Cruise Prop. \\ \hline
   \ 0.4 \                  &  11.36              &   1                 & 94\%    \\
    0.5                   &  9.28               &   3                 & 90\%     \\
    0.6                   &  7.97               &   6                 & 85\%    \\
    0.7                   &  7.12               &   16                & 67\%     \\
    0.8                   &  6.63               &   19                & 50\%    \\
    0.9                   &  6.37               &   26                & 36\%     \\
    1.0                   &  6.23               &   31                & 22\%    \\
    1.1                   &  6.15               &   37                & 16\%     \\
    1.2                   &  6.11               &   44                & 7\%     \\
    1.3                   &  6.10               &   45                & 0\%     \\
    \hline
  \end{tabular}
\end{table}

\begin{figure}[t]
      \centering
      \includegraphics[scale=0.6]{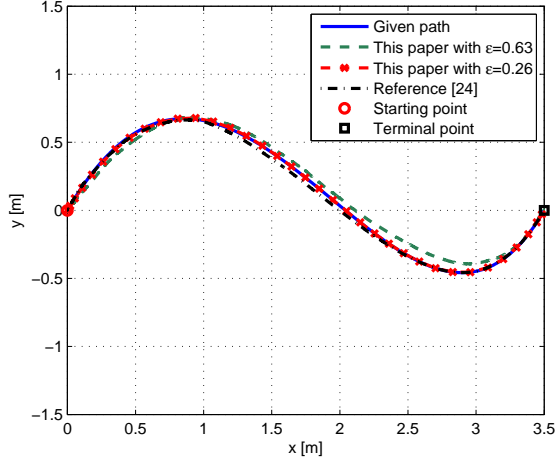}
      \caption{The blue solid curve represents a given path from the starting point to the terminal point. The black dash-dot curve represents the real path of following the trajectory obtained by \cite{2000_JRS}. The green dash curve and red dash curve with star marker represent real paths of following trajectories obtained by the proposed algorithm with $\varepsilon=0.63$ and $\varepsilon=0.26$, respectively.}
      \label{fig_exp_path}
\end{figure}

\begin{figure}[t]
      \centering
      \includegraphics[scale=0.6]{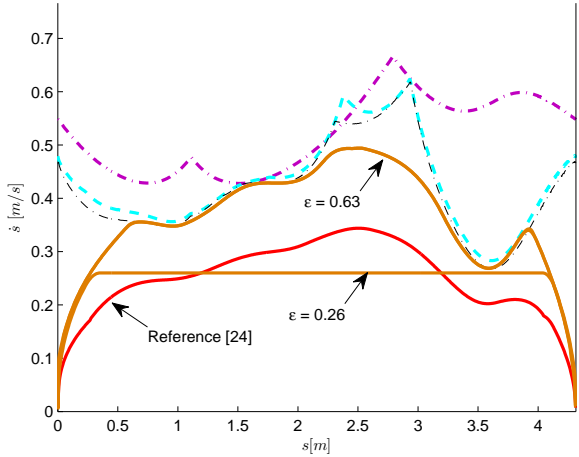}
      \caption{The orange solid curves are feasible trajectories obtained by the proposed algorithm with $\varepsilon=0.63$ and $\varepsilon=0.26$, respectively. The red solid curve is a piecewise polynomial feasible trajectory obtained by the method \cite{2000_JRS}.}
      \label{fig_alg_traj}
\end{figure}

\begin{table}[!htb]
  \centering
  \caption{Comparative results on the computational time [\textsc{ms}]}
  \label{tab_2}
  \begin{tabular}{|c|c|c|}
    \hline
    Normalized parameter         &  This paper     &   Reference \cite{2011_JIRS}  \\ \hline
   \ 0 \                  &  1                &   221                     \\
    0.1                   &  5               &   245                      \\
    0.2                   &  9               &   239                     \\
    0.3                   &  21               &   247                    \\
    0.4                   &  22               &   232                  \\
    0.5                   &  24               &   201                    \\
    0.6                   &  36               &   242                    \\
    0.7                   &  42               &   246                     \\
    0.8                   &  45               &   233                    \\
    0.9                   &  48               &   251                     \\
    1.0                   &  59               &   233                     \\
    \hline
  \end{tabular}
\end{table}

The adjustment of the functional parameter $\varepsilon$ affects the traveling time, computational time and cruise proportion of trajectories generated by the proposed algorithm. Fig. \ref{fig_sim_multitraj} shows that when $\varepsilon$ increases from 0.4 to 1.3, the blue solid line with star marker $M(s)$ moves up. Accordingly, the traveling time and cruise proportion of trajectories (orange solid curves) generated by the proposed algorithm decrease. In addition, Tab. \ref{tab_1} lists the detailed information of the traveling time, computational time and cruise proportion for each $\varepsilon$, which verifies that \emph{Theorem \ref{thm_2}-\ref{thm_4}} of the proposed algorithm hold.

\subsection{Comparative Experimental Results}

\begin{figure}[t]
      \centering
      \includegraphics[scale=0.6]{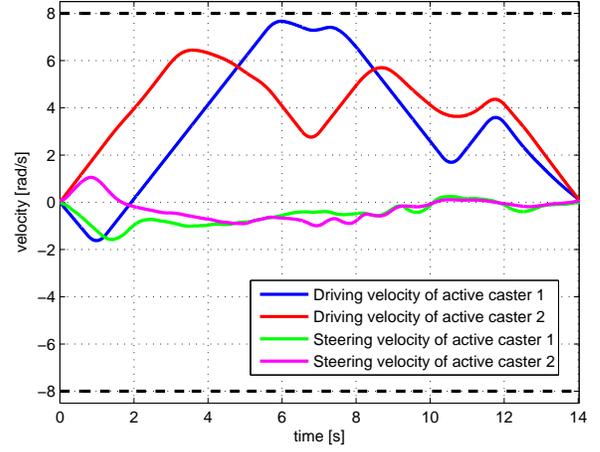}
      \caption{The driving and steering velocities of active casters of the OWMR following the feasible trajectory obtained by the proposed algorithm with $\varepsilon=0.63$.}
      \label{fig_vel_2}
\end{figure}

\begin{figure}[t]
      \centering
      \includegraphics[scale=0.6]{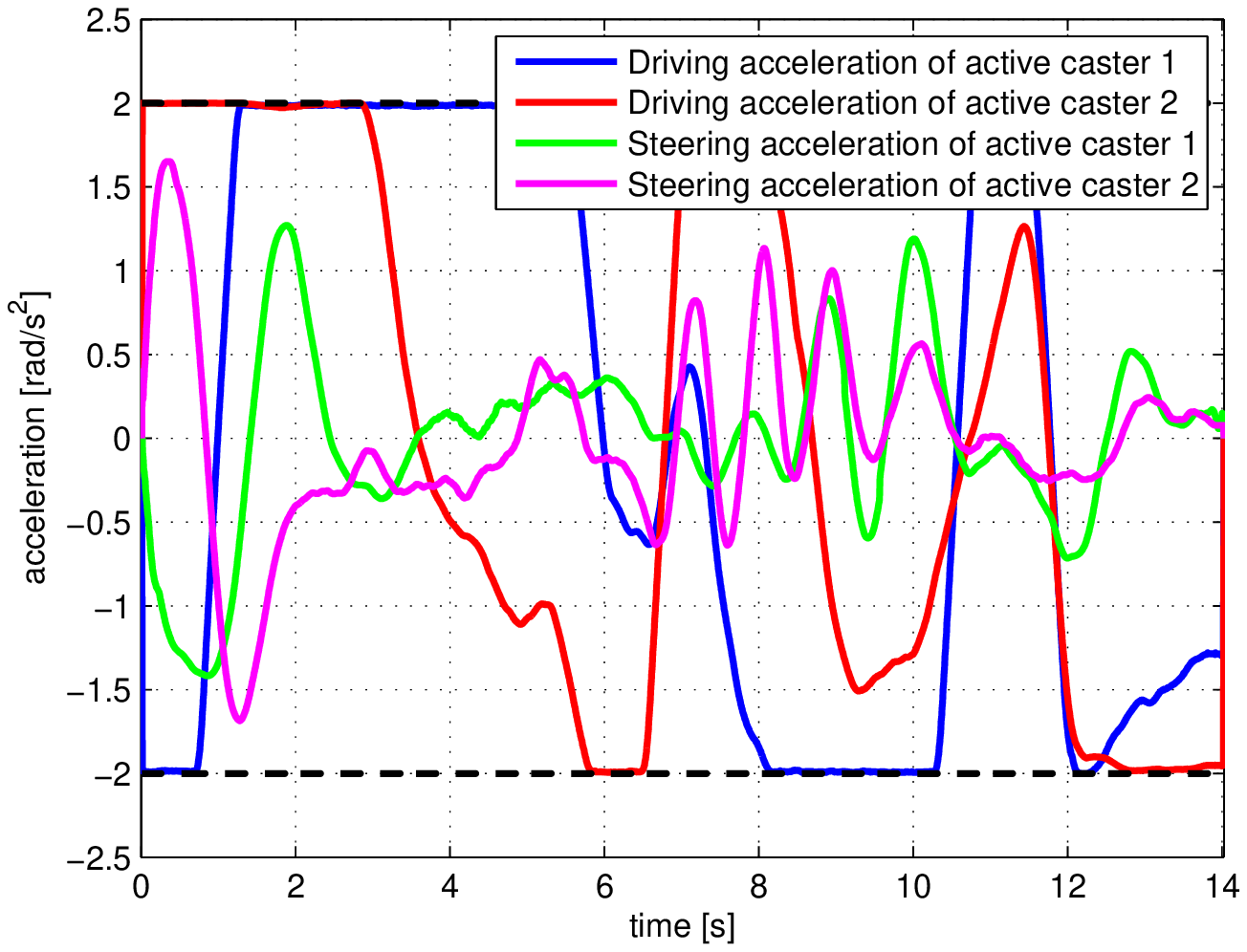}
      \caption{The driving and steering accelerations of active casters of the OWMR following the feasible trajectory obtained by the proposed algorithm with $\varepsilon=0.63$.}
      \label{fig_acc_2}
\end{figure}

\begin{figure}[t]
      \centering
      \includegraphics[scale=0.6]{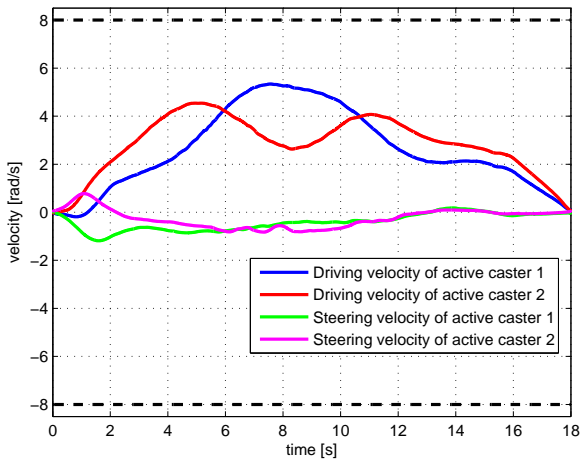}
      \caption{The driving and steering velocities of active casters of the OWMR following the feasible trajectory \cite{2000_JRS}.}
      \label{fig_vel_3}
\end{figure}

\begin{figure}[t]
      \centering
      \includegraphics[scale=0.6]{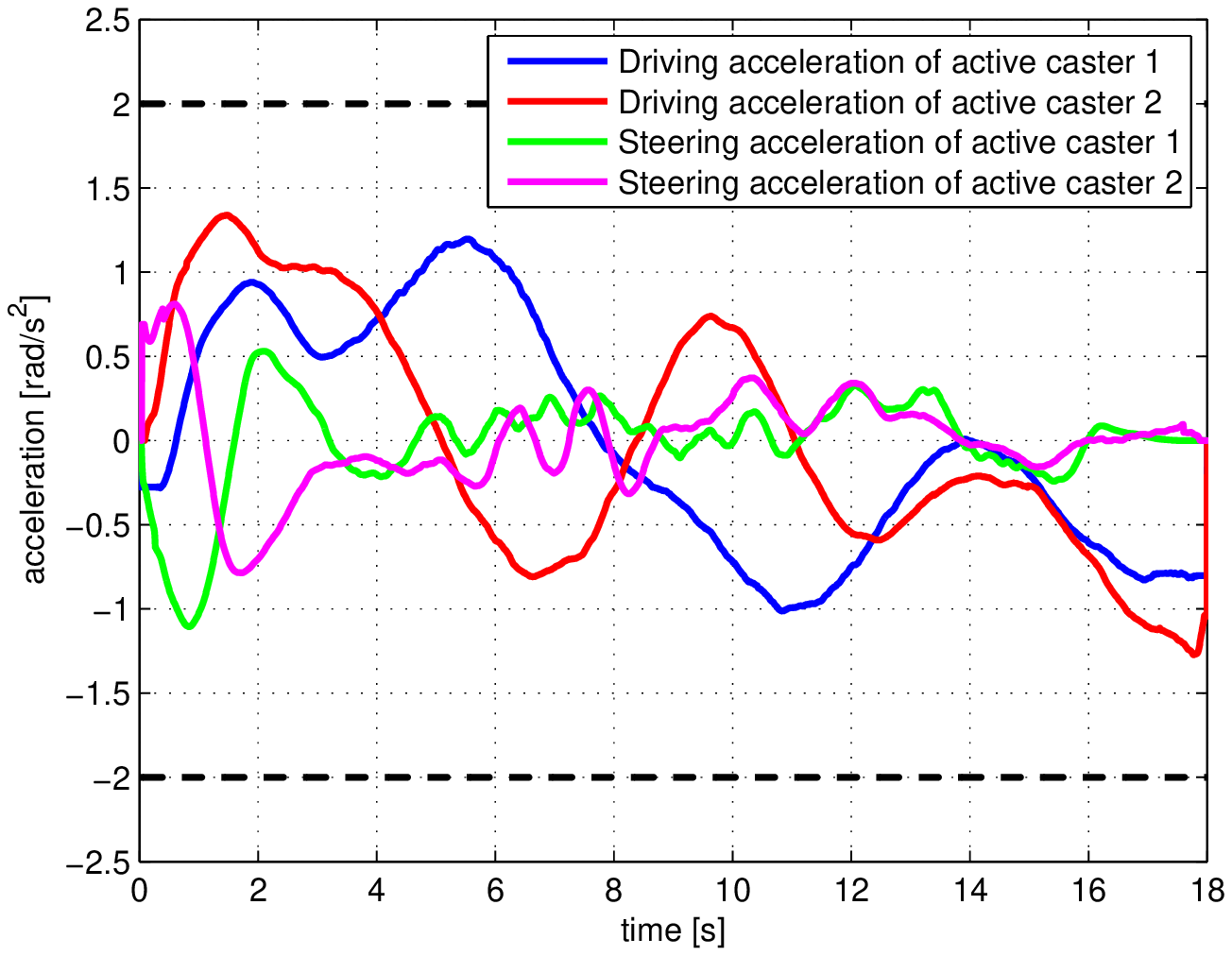}
      \caption{The driving and steering accelerations of active casters of the OWMR following the feasible trajectory \cite{2000_JRS}.}
      \label{fig_acc_3}
\end{figure}

In order to show the real-time performance of the proposed algorithm, and the advantages of the tradability mechanism between time-optimal motion and cruise motion in the proposed algorithm, this subsection provides comparative experimental results with existing methods \cite{2000_JRS} and \cite{2011_JIRS}.

Both works in \cite{2000_JRS,2011_JIRS} transform a path-constrained trajectory planning problem into a nonlinear optimization problem, and then solve the problem by numerical optimization techniques (such as FTM or SQP in $\mathrm{MatLab}^{\mathrm{TM}}$). In the nonlinear optimization problem, the trajectory of robotic systems is expressed as cubic or fifth-order polynomials. In order to output smooth trajectories, the jerk of trajectories is bounded and described as inequality constraints in the optimization problem. For objective functions in the optimization problem, the work in \cite{2000_JRS} selects the traveling time of trajectories, while the work in \cite{2011_JIRS} selects the weighted sum of the traveling time and the integral of the squared jerk. By adjusting the weights of two terms in the objective function, the method \cite{2011_JIRS} can output more smooth but slower trajectories, or faster but less smooth trajectories. In following experimental cases, the comparative results with \cite{2000_JRS} show the advantages of the tradability mechanism between cruise and time-optimal motions in the proposed algorithm, while the comparison results with \cite{2011_JIRS} validates the real-time performance of the proposed algorithm.

\emph{Experimental Verification of Time-optimal Motion with a Large $\epsilon$.} The given path of OWMR is set as a blue cubic B\`{e}zier curve in Fig. \ref{fig_exp_path}. Each element of velocity constraints $\bm v_{max}$ is set as $8[rad/s]$, and each element of acceleration constraints $\bm a_{max}$ is set as $2[rad/s^2]$. Both starting and terminal path velocities are set as zero. According to (\ref{eq_mvc_star}), the curve $MVC(s)$ is equal to $MVC^*(s)$ when the functional parameter $\varepsilon$ is set as $\mathrm{Max}(MVC)=0.63$. As shown in Fig. \ref{fig_alg_traj}, the proposed algorithm and the method \cite{2000_JRS} generate smooth and feasible trajectories, respectively. Moreover, the path velocity of the feasible trajectory obtained by the proposed algorithm is greater than or equal to the method \cite{2000_JRS} at each path coordinate. With the aid of a simple PID controller, the OWMR tracks these feasible trajectories. Fig. \ref{fig_vel_2}-\ref{fig_acc_3} show the velocity and acceleration of active casters of the OWMR following the feasible trajectories, which also verify the time optimality of the proposed algorithm.

\begin{figure}[t]
      \centering
      \includegraphics[scale=0.6]{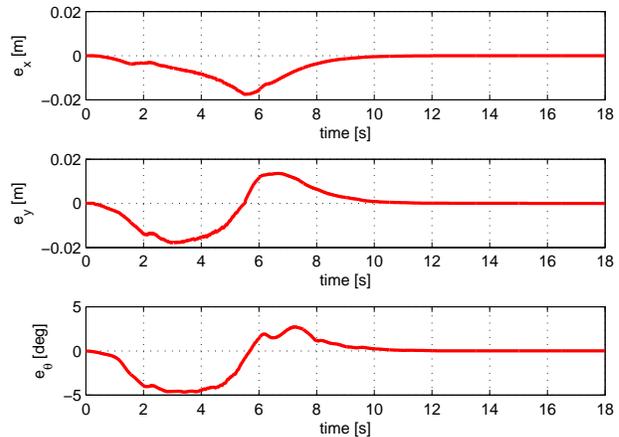}
      \caption{The tracking error for the position and orientation of the OWMR following the feasible trajectory obtained by the proposed algorithm with $\varepsilon=0.26$.}
      \label{fig_error_4}
\end{figure}

\begin{figure}[t]
      \centering
      \includegraphics[scale=0.6]{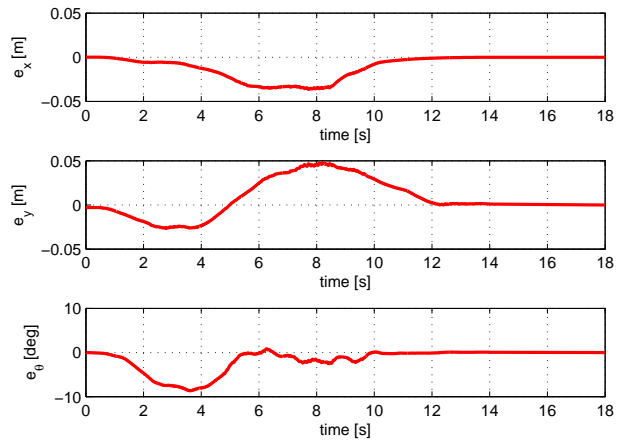}
      \caption{The tracking error for the position and orientation of the OWMR following the feasible trajectory \cite{2000_JRS}.}
      \label{fig_error_3}
\end{figure}

\emph{Experiment Verification of Cruise Motion with a Small $\epsilon$.} When the functional parameter $\varepsilon$ decreases to 0.26, the proposed algorithm generates another feasible trajectory of which the traveling time is the same with the trajectory \cite{2000_JRS}. As shown in Fig. \ref{fig_alg_traj}, the feasible trajectory with $\varepsilon=0.26$ possesses a high proportion of cruise motion, which improves tracking accuracy and makes robots move steadily. Fig. \ref{fig_error_4} and Fig. \ref{fig_error_3} show that the tracking accuracy of the position and orientation of the OWMR following the feasible trajectory obtained by the proposed algorithm with $\varepsilon=0.26$ is better than the feasible trajectory obtained by the method \cite{2000_JRS}. In addition, with respect to \cite{2000_JRS}, the velocity and acceleration of active casters of the OWMR following the feasible trajectory obtained by the proposed algorithm are more smooth as shown in Fig. \ref{fig_vel_4} and Fig. \ref{fig_acc_4}. It verifies that the proposed algorithm can output a feasible trajectory possessing a high proportion of cruise motion, which improve the tracking accuracy and make robots move steadily.

\begin{figure}[t]
      \centering
      \includegraphics[scale=0.6]{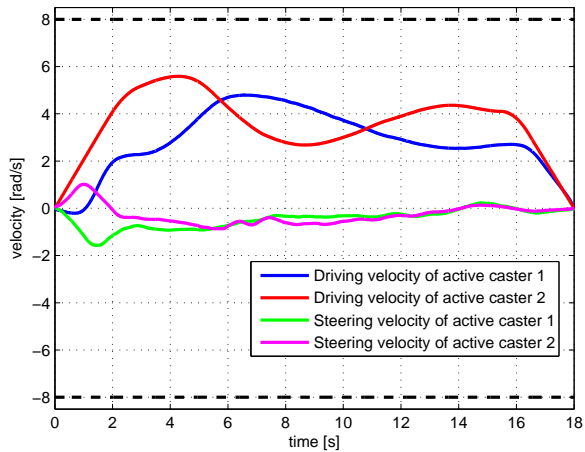}
      \caption{The driving and steering velocities of active casters of the OWMR following the feasible trajectory obtained by the proposed algorithm with $\varepsilon=0.26$.}
      \label{fig_vel_4}
\end{figure}

\begin{figure}[t]
      \centering
      \includegraphics[scale=0.6]{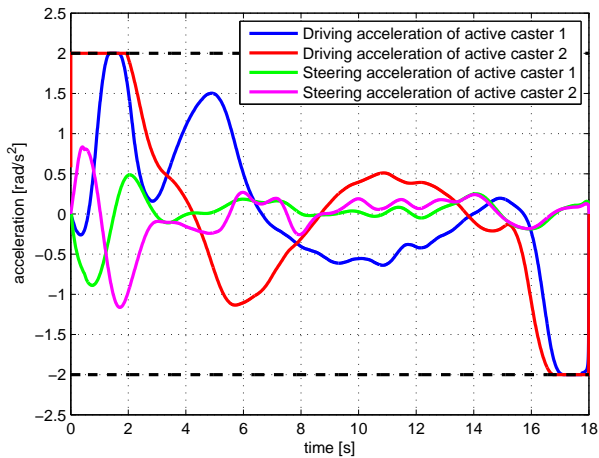}
      \caption{The driving and steering accelerations of active casters of the OWMR following the feasible trajectory obtained by the proposed algorithm with $\varepsilon=0.26$.}
      \label{fig_acc_4}
\end{figure}

\emph{Experiment Verification of the Real-time Performance.} Although the method \cite{2011_JIRS} and the proposed algorithm possess own adjustable mechanisms, the computational time of the proposed algorithm is also adjustable and it can achieve real-time performance. Tab. \ref{tab_2} shows comparative results on the computational time between the method \cite{2011_JIRS} and the proposed algorithm. For consistency, the weight parameter in \cite{2011_JIRS} and the functional parameter in the proposed algorithm are normalized as `0 to 1'. Tab. \ref{tab_2} shows that the trajectory generation of the proposed algorithm is much faster than the method \cite{2011_JIRS}, as well as the computational time of the proposed algorithm is in proportional to $\varepsilon$. In addition, compared with \cite{2011_JIRS}, the proposed approach can achieve cruise motion and adjust the cruise motion speed explicitly in the planning module.

\section{Conclusion}

We have proposed a real-time and acceleration-continuous trajectory planning algorithm along given paths. After smoothing time-optimal trajectories, it is proven that the proposed algorithm still preserves the completeness property. In addition, the proposed algorithm possesses a built-in tradability mechanism which flexibly changes the traveling time, cruise proportion and computational time of the resultant trajectory by adjusting a user-specified functional parameter. We have also provided the detailed proofs of the tradability mechanism in terms of the traveling time, cruise proportion and computational time. Simulation and comparative experimental results on omnidirectional wheeled mobile robots demonstrate the superior performance of the proposed algorithm.


\ifCLASSOPTIONcaptionsoff
  \newpage
\fi

\end{document}